\documentclass{article}

\usepackage[numbers]{natbib}

    \usepackage[preprint]{neurips_2021}

\usepackage[utf8]{inputenc} 
\usepackage[T1]{fontenc}    
\usepackage{hyperref}       
\usepackage{url}            
\usepackage{booktabs}       
\usepackage{amsfonts}       
\usepackage{nicefrac}       
\usepackage{microtype}      
\usepackage{xcolor}         
\usepackage{bbm}
\usepackage{graphicx}
\usepackage{amsmath,amssymb,amsthm,amsfonts}

\usepackage{paralist}
\usepackage{bm}
\usepackage{xspace}
\usepackage{url}
\usepackage{prettyref}
\usepackage{boxedminipage}
\usepackage{wrapfig}
\usepackage{ifthen}
\usepackage{color}
\usepackage{xcolor}
\usepackage{xspace}
\usepackage{algorithm}
\usepackage{thmtools}
\usepackage{thm-restate}

\newcommand{\midd}{{\sf mid}}

\usepackage{amsmath,amsthm,amsfonts,amssymb}
\usepackage{mathtools}
\usepackage{algorithm,algpseudocode}
\usepackage{hyperref}
\usepackage{graphicx}

\usepackage{algpseudocode}

\usepackage{fullpage}
\usepackage{hyperref}
\usepackage{cleveref}
\usepackage{nicefrac}

\newtheorem{theorem}{Theorem}
\newtheorem{claim}{Claim}
\newtheorem{lemma}{Lemma}
\newtheorem*{lemma*}{Lemma}
\newtheorem{definition}{Definition}
\newtheorem{corollary}{Corollary}

\newtheorem{remark}{Remark}

\DeclareMathOperator*{\argmin}{arg\,min}
\usepackage{subcaption}

\usepackage[utf8]{inputenc}

\newcommand{\rank}{\mathsf{rank}}
\newcommand{\tr}{\mathsf{Tr}}
\newcommand{\opt}{\mathsf{OPT}}
\newcommand{\rr}{\textsc{R}\space}
\newcommand{\alg}{\textsf{Alg}\space}
\newcommand{\sd}{\textsf{sd}_\lambda}
\newcommand{\lblq}{\mathfrak{lq} (X_1)}

\newcommand{\sign}{\textsf{sgn}}

\usepackage{thmtools}
\usepackage{thm-restate}

\title{Semi-supervised Active Regression}

\author{%
  Fnu Devvrit\thanks{Equal contribution, listed alphabetically} \\
  Department of Computer Science\\
  University of Texas at Austin
  \And
  Nived Rajaraman$^*$ \\
  Department of Electrical Engineering and Computer Sciences\\
  University of California, Berkeley
  \And
  Pranjal Awasthi\\
  Google Research \& Department of Computer Science\\
  Rutgers University
}

\begin{document}

\maketitle

\begin{abstract}
    Labelled data often comes at a high cost as it may require recruiting human labelers or running costly experiments. 
    At the same time, in many practical scenarios, one already has access to a partially labelled, potentially biased dataset that can help with the learning task at hand. Motivated by such settings, we formally initiate a study of \textit{semi-supervised active learning} through the frame of linear regression. In this setting, the learner has access to a dataset $X \in \mathbb{R}^{(n_1+n_2) \times d}$ which is composed of $n_1$ unlabelled examples that an algorithm can actively query, and $n_2$ examples labelled a-priori.
    Concretely, denoting the true labels by $Y \in \mathbb{R}^{n_1 + n_2}$, the learner's objective is to find $\widehat{\beta} \in \mathbb{R}^d$ such that,
    \begin{equation}
        \| X \widehat{\beta} - Y \|_2^2 \le (1 + \epsilon) \min_{\beta \in \mathbb{R}^d} \| X \beta - Y \|_2^2
    \end{equation}
    while making as few additional label queries as possible. In order to bound the label queries, we introduce an instance dependent parameter called the reduced rank, denoted by $\rr_X$, and propose an efficient algorithm with query complexity $O(\rr_X/\epsilon)$. This result directly implies improved upper bounds for two important special cases: (i) active ridge regression, and (ii) active kernel ridge regression, where the reduced-rank equates to the \textit{statistical dimension}, $\sd$ and \textit{effective dimension}, $d_\lambda$ of the problem respectively, where $\lambda \ge 0$ denotes the regularization parameter. For active ridge regression we also prove a matching lower bound of $O(\sd / \epsilon)$ on the query complexity of any algorithm. This subsumes prior work that only considered the unregularized case, i.e., $\lambda = 0$.
\end{abstract}

\section{Introduction}
Classification and regression form the cornerstone of machine learning. Often these tasks are carried out in a supervised learning manner - a learner collects many examples and their labels, and then runs an optimization algorithm to minimize a loss function and learn a model. However, 
the process of collecting labelled data comes at a cost, for instance when a human labeller is required, or if an expensive experiment needs to be run for the same. Motivated by such scenarios, two popular approaches have been proposed to mitigate these issues in practice: Active Learning \cite{settles2009active} and Semi-Supervised Learning \cite{zhu2005semi}. In Active Learning, the learner carries out the task after adaptively querying the labels of a small subset of characteristic data points. On the other hand, semi-supervised learning is motivated by the fact that learners often have access to massive amounts of unlabelled data in addition to some labelled data, and algorithms leverage both to carry out the learning task. Active learning and Semi-supervised learning have found numerous applications in areas such as text classification \cite{10.1162/153244302760185243}, visual recognition \cite{1334570} and foreground-background segmentation \cite{konyushkova2015introducing}.

In this work, we study an intersection of the above two approaches called \textit{semi-supervised active learning}. In this setting, the learner is provided a-priori labelled data $X_2 \in \mathbb{R}^{n_2 \times d}$, as well as a dataset of unlabelled data $X_1 \in \mathbb{R}^{n_1 \times d}$, points of which can be submitted to an oracle to be labelled each at unit cost. 
We study the problem through the framework of linear regression, where the objective of the learner is to minimize the squared-loss $\| X \beta - Y \|_2^2$ while making as few additional label queries (of points in $X_1$) as possible. Here $X$ is the overall dataset in $\mathbb{R}^{(n_1 + n_2) \times d}$ and $Y \in \mathbb{R}^{n_1 + n_2}$ denotes the corresponding labels. The $\epsilon$-query complexity is defined as the number of additional labels (in $X_1$) an algorithm queries to compute a $\widehat{\beta} \in \mathbb{R}^d$ such that,
\begin{equation} \label{eq:objective}
    \| X \widehat{\beta} - Y \|_2^2 \le (1 + \epsilon) \min_{\beta \in \mathbb{R}^d} \| X \beta - Y \|_2^2.
\end{equation}
The above framework is quite flexible and captures as special cases, standard formulations of the active regression problem.
Indeed, when no labelled data is provided (i.e. $X_2$ is empty), semi-supervised active regression reduces to active linear regression, that has been studied in several previous works \cite{drineas2007relativeerror,chen2019active,derezinski2018unbiased,derezinski2018leveraged,Boutsidis_2013}. 

Another important special case of semi-supervised active learning is when the labelled dataset $X_2$ is initialized as the standard basis vectors scaled by a factor of $\sqrt{\lambda}$, with labels as $0$. The squared-loss of $\beta$ for this dataset exactly corresponds to the ridge regression loss: $\| X_1 \beta - Y_1 \|_2^2 + \lambda \| \beta \|_2^2$ where $Y_1$ are the labels of points in $X_1$.

One can also capture the active kernel ridge regression problem in this formulation. 
Here, the data points lie in a {\em reproducing kernel hilbert space} (RKHS) \cite{paulsen_raghupathi_2016,4060955,6524743} 
that is potentially infinite dimensional, but the learner has access to the kernel matrix of the data points. Denoting the kernel matrix by $K \in \mathbb{R}^{N \times N}$ and the labels of the points as $Y \in \mathbb{R}^N$, the objective in kernel ridge regression is to minimize the loss $\| K \beta - Y \|_2^2 + \lambda \| \beta \|_K^2$ 
This is special case of the semi-supervised active regression framework by choosing the unlabelled dataset as $X_2 = \sqrt{K} \in \mathbb{R}^{N \times N}$ with each point labelled by $0$, and the labelled dataset $X_1 = K \in \mathbb{R}^{N \times N}$ with labels $Y \in \mathbb{R}^N$.

Our first main contribution is to propose an instance dependent parameter called the {\em reduced rank}, denoted $\rr_X$ that upper bounds the $\epsilon$-query complexity of semi-supervised active regression. Intuitively, $\rr_X$ measures the relative importance of a-priori labelled points $X_2$ in comparison with the overall dataset. Formally, the reduced rank is defined as:
\begin{align} \label{eq:rr-def}
    \rr_X = \tr \left( \left( X_1^T X_1 + X_2^T X_2 \right)^{-1} X_1^T X_1 \right)
\end{align}
We then propose a polynomial time algorithm that solves semi-supervised active linear regression with $\epsilon$-query complexity bounded by $O(\rr_X/\epsilon)$. In the special cases of active ridge regression and kernel ridge regression, $\rr_X$ is shown to be the equal to the \textit{statistical dimension} $\sd$ \cite{avron2017sharper} and the \textit{effective dimension} $d_\lambda$ \cite{pmlr-v97-yasuda19a} of the problem respectively, which implies black-box query complexity and algorithmic guarantees for the same. These parameters are formally defined in \Cref{sec:leverage-score}.

\noindent \textbf{Tight bounds for active ridge regression.} The recent work of \cite{chen2019active} studies the problem of active regression and following a long line of work~(see \cite{drineas2007relativeerror}) provides a near optimal~(up to constant factors) algorithm for active regression with query complexity of $O(d/\epsilon)$. One can also use the algorithm of \cite{chen2019active} to solve the active ridge regression problem~($\lambda \neq 0$) and achieve a query complexity of $O(d/\epsilon)$. However, the dependence on $d$ is not ideal in this setting as often the statistical dimensions $\sd$ of the data is the right measure of the complexity of the problem. As a consequence of our general algorithm, we obtain an algorithm for active ridge regression with query complexity $O(\sd/\epsilon)$, and we further show that the achieved bound is tight. As a result we resolve the complexity of ridge regression in the active learning setting.



\noindent \textbf{Improved bounds for active kernel ridge regression}. In the context of kernel ridge regression, the work of \cite{alaoui2015fast} analyzed kernel leverage score sampling (and in general Nyström approximation based methods) showing that the sampling approach makes $O(d_\lambda \log (N))$ label queries to find a constant factor approximation to the optimal loss. In contrast, our algorithm when applied to kernel ridge regression achieves a query complexity guarantee of $O(d_\lambda / \epsilon)$ showing that the dependence on the number of data points, $N$, can be completely eliminated, which can be very large in practical settings.

\noindent \textbf{Techniques.} Our algorithm is based on proposing a variant of the Randomized BSS {\em spectral sparsification} 
algorithm of \cite{lee2015constructing,batson2009twiceramanujan}, 
that returns a small weighted subset of the dataset, and subsequently carries out weighted linear regression under squared-loss on this dataset. The key technical contribution of this work is to construct a variant of the spectral sparsification algorithm with an upper bound guarantee on the expected number of points sampled from the unlabeled set. We discuss the algorithm and its guarantees in more detail in \Cref{sec:alg}.


\paragraph{Robustness to biased labels.} From a practical point of view, several algorithms for active linear regression and classification have been proposed \cite{drineas2007relativeerror,chen2019active,derezinski2018unbiased,derezinski2018leveraged,Boutsidis_2013,10.1162/153244302760185243,1334570,konyushkova2015introducing,https://doi.org/10.1002/widm.1132,DASGUPTA20111767,10.1145/1390156.1390183,beygelzimer2009importance}. 
However, in the face of bias in the labelled dataset, as we discuss in \Cref{sec:bias}, existing approaches such as uncertainty sampling \cite{https://doi.org/10.1002/widm.1132} or the active perceptron algorithm \cite{DASGUPTA20111767} may suffer from very slow convergence. In contrast, our algorithm treats labelled and unlabelled data equally while submitting label queries, and yet achieves the optimal worst case query complexity guarantees. See Section~\ref{sec:bias} for more details.



The rest of the paper is organized as follows. We cover related work in \Cref{sec:rel-work} and introduce relevant preliminaries in \Cref{sec:preliminaries}. We motivate and present the new parameter $\rr_X$ in \Cref{sec:leverage-score} and present our algorithm for semi-supervised active linear regression and a proof sketch in \Cref{sec:alg}. In \Cref{sec:lower-bound} we sketch a lower bound for active ridge regression. Finally, we discuss pitfalls of existing active learning algorithms extended to the semi-supervised setting in the face of bias in the labelled data in \Cref{sec:bias}, and conclude the paper in \Cref{sec:conclusion}.






\vspace{-0.5em}
\subsection{Related Work}\label{sec:rel-work}
\vspace{-0.1em}

Many existing works use active learning to augment semi-supervised learning based approaches \cite{gao2020consistencybased,Zhu03combiningactive,drugman2019active,RHEE2017109,sener2018active}. The key in several of these works is to use active learning methods to decide which labels to sample, and then use semi-supervised learning to finally learn the model. 



In the context of active learning for linear regression, \cite{drineas2007relativeerror,Woodruff_2014} analyze the leverage score sampling algorithm showing a sample complexity upper bound of $O(\nicefrac{d \log(d)}{\epsilon})$. Subsequently for the case of $\epsilon \ge d+1$, \cite{derezinski2018unbiased} show that volume sampling achieves an $O(d)$ sample complexity. Later, \cite{derezinski2018leveraged} show that rescaled volume sampling matches the sample complexity of $O(\nicefrac{d \log(d)}{\epsilon})$ achieved by leverage score sampling. Finally, \cite{chen2019active} show that a spectral sparsification based approach using the Randomized BSS algorithm \cite{lee2015constructing} achieves the optimal sample complexity of $O(\nicefrac{d}{\epsilon})$ for the problem.

In the context of kernel ridge regression, the works of \cite{friedman,pmlr-v97-yasuda19a,musco2017recursive} study the problem of reducing the runtime and number of kernel evaluations to approximately solve the problem. These results show that the number of kernel evaluations can be made to scale only linearly with the effective dimension of the problem, $d_\lambda (K)$ up to log factors. Recently, \cite{alaoui2015fast} prove a statistical guarantee showing that any optimal Nystrom approximation based approach samples at most $d_\lambda (K) \log (N)$ labels to find a constant factor approximation to the kernel regression loss. 

A closely related problem to active learning is the coreset (and weak coreset) problem \cite{andoni2020streaming,munteanu2018coresets}, where the objective is to find a low-memory data structure that can be used to approximately reconstruct the loss function which naïvely requires storing all the training examples to compute.
\cite{pmlr-v108-kacham20a} study the weak coreset problem for ridge regression and propose an algorithm that returns a weak coreset of $O(\sd/\epsilon)$ points which can approximately recover the optimal loss. A drawback of their algorithm is that the labels of all the points in the dataset are required, which may come at a very high cost in practice. As we discuss in \Cref{sec:alg}, our algorithm also returns a weighted subset of $O(\sd/\epsilon)$ points in $\mathbb{R}^d$ and a set of $d$ additional weights that can be used to reconstruct a $(1+\epsilon)$-approximation for the original ridge regression instance. In terms of memory requirement, this matches the result of \cite{pmlr-v108-kacham20a}; however it suffices for our algorithm to query \textit{at most $O(\sd/\epsilon)$ labels}.
\section{Preliminaries} \label{sec:preliminaries}
\vspace{-0.8em}
In this section, we formally define semi-supervised active regression under squared-loss. The learner is provided datasets $X_1 \in \mathbb{R}^{n_1 \times d}$ (comprised of $n_1$ points in $\mathbb{R}^d$) and $X_2 \in \mathbb{R}^{n_2 \times d}$. The labels of the points are respectively denoted $Y_1\in \mathbb{R}^{n_1}$ and $Y_2 \in \mathbb{R}^{n_2}$ and define $X = \begin{bmatrix}
    X_1\\
    X_2
\end{bmatrix}$ and $Y = \begin{bmatrix}
    Y_1\\
    Y_2
\end{bmatrix}$. In the agnostic setting, the labels of the points are not assumed to be generated from an unknown underlying linear function and can be arbitrary. We study the problem in the overconstrained setting with $n_1 + n_2 \ge d$. Furthermore, in the active setting, it is typically the case that $n_1 \gg d$. Whenever understood from the context, we use $X$ to represent the dataset in set notation.

We assume that the learner is a-priori provided $Y_2$ and hence knows the labels of all points in $X_2$, but the labels of points in $X_1$ are not known. However, the learner can probe an oracle to return the label of any queried point in $X_1$ and incurs a unit cost for each label returned. The objective of the learner is to return a linear function $\beta \in \mathbb{R}^d$ approximately minimizing the squared $\ell_2$ loss: $\| X \widehat{\beta} - Y \|_2^2 \le (1 + \epsilon) \| X \beta^* - Y \|_2^2$ where the optimal linear regression function and its value are:
\begin{align} \label{eq:lrobj}
    \beta^* \triangleq \argmin_{\beta \in \mathbb{R}^d} \| X \beta - Y \|_2^2 , \qquad \opt \triangleq \| X \beta^* - Y \|_2^2.
\end{align}
The semi-supervised active regression formulation has the following two important problems as subcases:

\noindent \textbf{Active ridge regression:} In the active ridge regression problem, the learner is provided an unlabelled dataset $X_1 \in \mathbb{R}^{n \times d}$ of $n$ points. The learner has access to an oracle which can be actively queried to label points in $X_1$. Representing the labels by $Y_1 \in \mathbb{R}^n$, the objective of the learner is to minimize $\| X_1 \beta - Y_1 \|_2^2 + \lambda \| \beta \|_2^2$. Defining $X = \begin{bmatrix}
X_1 \\
\sqrt{\lambda} I
\end{bmatrix} \in \mathbb{R}^{(n+d) \times d}$ and $Y = \begin{bmatrix}
Y_1 \\ 0
\end{bmatrix} \in \mathbb{R}^{n+d}$, observe that $\| X \beta - Y \|_2^2 = \| X_1 \beta - Y_1 \|_2^2 + \lambda \| \beta \|_2^2$. Thus, the ridge regression objective is a special case of the linear regression objective in \cref{eq:lrobj}.

\noindent \textbf{Active kernel ridge regression:} In the kernel ridge regression problem, the learner operates in a reproducing kernel Hilbert space $\mathcal{H}$ and is provided an unlabelled dataset $X$ 
of $n$ points with (implicit) feature representations $\Phi (x)$ for $x \in X$. The objective is to learn a linear function $F(x)$, in the feature representations of points in $A$ that minimizes the squared $\ell_2$ norm with regularization. Namely, $\sum_{i=1}^n (F(x_i) - Y_i)^2 + \lambda \| F \|_{\mathcal{H}}^2$. By the Representer Theorem \cite{kimeldorf1971some}, it turns out the optimal regression function $F^* (\cdot)$ can be expressed as $\sum_{i=1}^n \beta_i \kappa (x_i, \cdot)$ where $\kappa (\cdot,\cdot)$ is the kernel function. The kernel ridge regression objective can be expressed in terms of the kernel matrix $K \triangleq \begin{bmatrix}
\kappa (x_i, x_j)
\end{bmatrix}_{i,j = 1}^n \in \mathbb{R}^{n \times n}$. In other words, denoting the labels by $Y_1 \in \mathbb{R}^n$ and $\| \beta \|_K^2 = \beta^T K \beta$, the learner's objective is to minimize the loss,
\begin{align}
    \| K \beta - Y_1 \|_2^2 + \lambda \| \beta \|_K^2 
\end{align}
while minimizing the number of labels of points in $X_1$ queried. This objective can be represented as:
\begin{equation} \label{eq:krr-reduction}
    \left\| \begin{bmatrix}
    K \\
    \sqrt{\lambda} \sqrt{K} 
    \end{bmatrix} \beta - \begin{bmatrix}
    Y_1 \\
    0
    \end{bmatrix}\right\|_2^2
\end{equation}
where $Z = \sqrt{K}$ is defined as the unique solution to $Z^T Z = K$. Yet again, it is a special case of the linear regression objective in \cref{eq:lrobj} with $X_1 = K$, $X_2 = \sqrt{K}$ and $Y_2 = 0$.

\vspace{-0.7em}
\section{Motivating the reduced rank $\rr_X$} \label{sec:leverage-score}
\vspace{-0.55em}
A natural algorithm for semi-supervised active regression is to use the popular approach of  leverage score sampling \cite{drineas2007relativeerror,pmlr-v108-kacham20a}. This approach uses the entire dataset $X$~(both labeled and unlabeled) to construct a diagonal weight matrix $W_S$: here $S$ represents the support of the diagonal and $\mathbb{E} [|S|] \triangleq m = O(\nicefrac{d \log (d)}{\epsilon})$ for some large constant $C$. The learner subsequently solves the problem: $\widehat{\beta} = \argmin_\beta \| W_S (X\beta-Y) \|_2^2$ which is a weighted linear regression problem on $|S|$ examples. The works of \cite{drineas2007relativeerror,pmlr-v108-kacham20a} show that resulting linear function produces a $(1+\epsilon)$ multiplicative approximation to the optimal squared-loss:
\begin{align}
    \|X \widehat{\beta} - Y\|_2^2 &\leq \left(1+\epsilon\right) \inf_{\beta \in \mathbb{R}^d} \|X\beta - Y\|_2^2 \label{eq:0001}
\end{align}
The weight matrix $W_S$ returned by leverage score sampling is constructed as follows:
Letting $U\Sigma V^T$ denote the singular value decomposition of $X$, the algorithm iterates over all the points in $X$ and includes a point $x$ in $S$ with probability $p(x) \triangleq \min \left\{ 1 , \frac{m}{d} \|U(x)\|_2^2 \right\}$, where $U(x)$ is the row in $U$ corresponding to point $x \in X$. The corresponding diagonal entry of $W_S$ is set as $\frac{1}{\sqrt{p(x)}}$. Since the points in $X_2$ are labelled, the number of times the algorithm queries the oracle to label a point is equal to $|S \cap X_1 |$. By definition of the sampling probabilities, this is proportional to the sum of the leverage scores across the points in $X_1$:
\begin{align}
    \mathbb{E}\left[\left| S\cap X_1 \right|\right] = \frac{m}{d} \sum_{x\in X_1} \|U(x)\|_2^2 = \frac{\log (d)}{\epsilon} \sum_{x\in X_1} \|U(x)\|_2^2 \label{eq:LS-pointssampled}
\end{align}
In \Cref{lemma:trD-vs-rr}, we prove that the term $\sum_{x \in X_1} \| U (x) \|_2^2$ can be expressed in terms of the reduced rank, i.e., $\rr_X$. Recall that $\rr_X$ is defined as $\tr \left( \left( X^T X \right)^{-1} X_1^TX_1 \right)$. Combining with \cref{eq:LS-pointssampled} results in the following upper bound on the number of additional label queries made by leverage score sampling for the semi-supervised active learning problem.

\begin{theorem}
For semi-supervised active linear regression, the number of additional labels (in $X_1$) queried by leverage score sampling is $O(\frac{\rr_X \cdot \log(d)}{\epsilon})$.
\end{theorem}

In order to further motivate $\rr_X$ we next turn to the ridge regression problem: Here, it turns out that $\rr_X$ equals the \textit{statistical dimension} of $X_1$, defined as:
\begin{equation} \label{eq:sd}
    \sd (X_1) \triangleq \sum\nolimits_{i=1}^{\rank (X_1)} \frac{1}{1 + \nicefrac{\lambda}{\sigma_i^2 (X_1)}},
\end{equation}
where $\sigma_i (X_1)$ is the $i^{th}$ largest singular value of $X_1$. This can be seen by plugging in $X_2$ as $\sqrt{\lambda} I$ in \cref{eq:rr-def} and expanding $X_1$ using its singular value decomposition as $U_1 \Sigma_1 V_1$. 

In comparison, in kernel ridge regression, $\rr_X$ evaluates to the \textit{effective dimension}, defined as:
\begin{equation} \label{eq:deff}
    d_\lambda (K) \triangleq  \sum\nolimits_{i=1}^{\rank (K)} \frac{1}{1 + \nicefrac{\lambda}{\sigma_i (K)}},
\end{equation}
where the $\sigma_i$'s are the eigenvalues of $K$. This yet again follows by plugging the choice of $X_1$ and $X_2$ in the kernel ridge regression setting from \cref{eq:krr-reduction}.

The common theme in the prior sampling based approaches is that they suffer from a logarithmic dependence in $d$ or $N$ which are often very large in practice. In this paper, we subsume previous results by showing an algorithm with an $O\left(\frac{\rr_X}{\epsilon}\right)$ query complexity to produce an approximate solution for semi-supervised active linear regression. This shows that it is possible to derive statistical guarantees for active ridge / kernel ridge regression that respectively depend \textit{only} on the statistical dimension / effective dimension of the problem.





\section{Algorithm}\label{sec:alg}


In this section, we design a polynomial time algorithm for semi-supervised active linear regression and prove that the $\epsilon$-query complexity of labels queried by the algorithm is $O(\frac{\rr_X}{\epsilon})$. Our algorithm (\Cref{alg:main}) follows the approach of \cite{chen2019active}, and is based on sampling a small subset of points adaptively, querying their labels and solving a weighted linear regression problem on the subsampled instance. The work of \cite{chen2019active} showed that if the sampling procedure is {\em well-balanced} (see definition below), then one can obtain a bound on the query complexity of the algorithm. The key then is to design a well balanced sampling procedure for the case of semi-supervised active linear regression. We present such a procedure in \Cref{alg:RBSS}, which is a variant of the spectral sparsification algorithm of \cite{lee2015constructing}. We then provide a proof sketch showing that \Cref{alg:RBSS} is indeed an $\epsilon$-well-balanced sampling procedure, and finally bound the query complexity of \Cref{alg:main}.



\begin{algorithm}[htb]
	\caption{\textsc{Semi-supervised active regression}}
	\label{alg:main}
	\begin{algorithmic}[1]
	    \State \textbf{Input:} $X \in \mathbb{R}^{n \times d}$; $Y \in \mathbb{R}^n$; accuracy parameter $\epsilon$; well balanced sampling procedure \alg
		\State Run $\alg (X,\epsilon)$ to return a subset of points $ \{ x_1,\cdots, x_m \}$ and weights $\{ w_1,\cdots,w_m \}$.
		\State Query the labels $y_i$ of $x_1,\cdots, x_m$ for $x_i \in X_1$. \Comment{Contributes to the additional label queries}
		\State \textbf{Return:} $\widehat{\beta} \gets \arg\min_{\beta \in \mathbb{R}^d} \sum_{i=1}^m w_i ( \beta^T x_i - y_i)^2$.
	\end{algorithmic}
\end{algorithm}

\begin{definition}[$\epsilon$-well balanced sampling. procedure]{\cite[Def 2.1]{chen2019active}}\label{def:eps-w-b}
Consider a randomized algorithm $\alg (X,\epsilon)$ that outputs a set of points $\{ x_1,\cdots, x_m \}$ and weights $\{ w_1,\cdots,w_m \}$ as follows: in each iteration $i = 1,\cdots,m$, the algorithm chooses a distribution $D_i$ over points in $X$ to sample $x_i \sim D_i$. Let $D$ be the uniform distribution over all points in $X$. Then, \alg is said to be an \textbf{$\epsilon$-well-balanced sampling procedure} if it satisfies the following two properties with probability $\ge \nicefrac{3}{4}$,
\begin{enumerate}
    \item The matrix $A = \textsf{diag} ( \{ \sqrt{w_i} \}_{i=1}^m) U \in \mathbb{R}^{m \times d}$ where $U \Sigma V^T$ is the SVD of $X$ satisfies:
    \begin{equation} \label{eq:spec-spar}
        \frac{3}{4} I \preceq A^T A \preceq \frac{5}{4} I
    \end{equation}
    Equivalently, For every $\beta \in \mathbb{R}^d$, $\sum_{i=1}^m w_i \langle \beta , x_i \rangle^2 \in \left[ \frac{3}{4}, \frac{5}{4} \right] \mathbb{E}_{x \sim D} \left[ \langle \beta , x \rangle^2 \right]$.
    \item For each $i \in [m]$, define $\alpha_i = \frac{D_i (x_i)}{D(x_i)} w_i$. Then, \alg must satisfy $\sum_{i=1}^m \alpha_i = O(1)$ and $\alpha_i K_{D_i} = O(\epsilon)$, where $K_{D_i}$ is the re-weighted condition number:
    \begin{equation} \label{eq:eps-wbs-2}
        K_{D_i} = \sup_x \left\{ \sup_{\beta \in \mathbb{R}^d} \left\{ \frac{D_i(x)}{D(x)} \cdot\frac{\langle \beta , x \rangle^2}{\mathbb{E}_{x' \sim D} \left[ \langle \beta , x' \rangle^2 \right]} \right\} \right\} 
    \end{equation}
\end{enumerate}
\end{definition}

The work of \cite{chen2019active} showed that given an $\epsilon$-well-balanced sampling procedure \alg, with probability $\nicefrac{3}{4}$, the linear function $\widehat{\beta}$ returned by \Cref{alg:main} satisfies $\|X\widehat{\beta}-Y\|_2^2 \leq \left(1+ O(\epsilon)\right) \min_{\beta \in \mathbb{R}^d}\|X\beta-Y\|_2^2$. Using our proposed sampling procedure in \Cref{alg:RBSS}, we now state our main theorem.


\begin{theorem} \label{thm:main}
For any $0 < \epsilon < 1$, with constant probability, \Cref{alg:main} queries the labels of $O\left(\frac{\rr_X}{\epsilon}\right)$ points in $X_1$, and outputs a diagonal weight matrix $W_S$ with support $S \subseteq [n_1+n_2]$ of size $|S| = O\left(\frac{d}{\epsilon}\right)$ such that
\begin{align}
    \|X\widehat{\beta}-Y\|_2^2 &\leq \left(1+O(\epsilon)\right)\min_{\beta \in \mathbb{R}^d}\|X\beta-Y\|_2^2 \label{eq:0001}
\end{align}
where $\widehat{\beta} = \argmin_{\beta \in \mathbb{R}^d} \left\| W_S(X\beta-Y) \right\|_2^2$
\end{theorem}



Our algorithm uses the $\epsilon$-well balanced sampling procedure (\Cref{alg:RBSS}) to outputs weights and simply returns $\widehat{\beta}$ as the solution to the weighted ERM problem defined by $\{ x_1,\cdots, x_m \}$ and $\{ w_1,\cdots, w_m \}$.

We now present \Cref{alg:RBSS} and show that it is an $\epsilon$-well-balanced sampling procedure. The algorithm is parameterized by $\gamma$ chosen to be $\triangleq \sqrt{\epsilon}/ C_0$ for a large constant $C_0 > 0$.

\begin{restatable}{theorem}{thmRBSS} \label{thm:RBSS}
\Cref{alg:RBSS} is an $\epsilon$-well-balanced sampling procedure, sampling $O\left(\frac{\rr_X}{\epsilon}\right)$ points in $X_1$.
\end{restatable}

\begin{algorithm}[htb]
	\caption{\textsc{ASURA} (Active semi-SUpervised Regression Algorithm)}
	\label{alg:RBSS}
	\begin{algorithmic}[1]
	    \State \textbf{Input:} $X \in \mathbb{R}^{n \times d}$.
		\State \textbf{Initialization:} $j=0$; $\gamma = \sqrt{\epsilon}/C_0$; $u_0=2d/\gamma$; $l_0=-2d/\gamma$; $A_0 = 0$.
		\While{$(u_j-l_j)+\sum_{i = 0}^{j-1} \Phi_i^{\mathrm{Id}} < 8d/\gamma$}
		\State Define $\Phi_{j}^{\mathrm{Id}} = \textsf{Tr} \left( (u_j I - A_j)^{-1} (A_j - l_j I)^{-1} \right)$.
		\State Set coefficient $\alpha_j' = \nicefrac{\gamma}{\Phi_{j}^{\mathrm{Id}}}$
		\State Sample a point from the multinomial distribution on $X$ which assigns probability to $x$ as,
		\begin{equation} \label{eq:pxj}
		    p_x \triangleq \frac{U(x)^T \left( (u_j I - A_j)^{-1} + (A_j-l_jI)^{-1} \right) U(x)}{\Phi_j^{\mathrm{Id}}}
		\end{equation}
		\Comment{Define the sampled point $x$ as $x_j$ and $p_x$ as $p_j$}
		\State Update $A_{j+1} \gets A_j + \frac{\gamma}{\Phi_{j}^{\mathrm{Id}}} \frac{1}{p_j} U(x_j) U(x_j)^T$
		\State Define the weight $w_j' \gets \frac{\gamma}{\Phi_{j}^{\mathrm{Id}}} \frac{1}{p_j}$
		\State Update $u_{j+1} \gets u_j + \frac{\gamma}{1-2\gamma} \frac{1}{\Phi_{j}^{\mathrm{Id}}}$ and $l_{j+1} \gets l_j + \frac{\gamma}{1+2\gamma} \frac{1}{\Phi_{j}^{\mathrm{Id}}}$.
		\State $j \gets j+1$
		\EndWhile
		\State Assign $m=j$,
		\State Define $\midd \triangleq \frac{u_m+l_m}{2}$ and for each $j$, set the coefficient $\alpha_j = \frac{\alpha_j'}{\midd}$ and the weight $w_j = \frac{w_j'}{\midd}$.
	    \State \textbf{Return} $\{ x_1,\cdots,x_m \}$; $\{ w_1,\cdots, w_m \}$.
	\end{algorithmic}
\end{algorithm}
The $\epsilon$-well-balanced sampling procedure we consider is a variant of the algorithm proposed in \cite{lee2015constructing} for spectral sparsification (\cref{eq:spec-spar}). In each iteration $j$ the algorithm maintains a matrix $A_j$ which is a proxy for $A^T A$ up to scaling (as defined in \Cref{def:eps-w-b}) and an upper and lower threshold $u_j$ and $l_j$ such that $l_j I \preceq A_j \preceq u_j I$. Points are sampled carefully in a way where $A_j$ never gets too close to the boundaries $l_jI$ and $u_j I$ and this is guaranteed by the adaptive sampling distribution in \cref{eq:pxj}. Over the iterations, the ratio $\nicefrac{u_j}{l_j}$, which controls the condition number of $A_j$ (and in turn the eigenvalues of $A^T A$) is made to approach $1$, until the condition \cref{eq:spec-spar} is satisfied.

First, notice the important difference in the stopping criterion of our algorithm compared to the Randomized BSS Algorithm of \cite{lee2015constructing}. This has important implications for bounding the total number of points sampled by our algorithm. The number of points sampled by the algorithm proposed in \cite{lee2015constructing} is $O\left(\nicefrac{d}{\gamma^2}\right)$ with constant probability, whereas \Cref{alg:RBSS} samples $O\left(\nicefrac{d}{\gamma^2}\right)$ points almost surely. 

We now briefly discuss the main challenges in proving \Cref{thm:RBSS}. The number of label queries made by the algorithm (of points in $X_1$), $\lblq$, is upper bounded by the number of iterations in which the algorithm samples a point in $X_1$. Indeed,
\begin{equation} \label{eq:m1}
    \lblq \le \sum\nolimits_{j=0}^{m-1} \sum\nolimits_{x\in X_1} p_x^{(j)}
\end{equation}
where $m$ is the total number of iterations the algorithm runs for, and $p_x^{(j)}$ is the probability of sampling point $x$ in the $j^{th}$ iteration as defined in \cref{eq:pxj}. By Wald's equation, $\mathbb{E}[ \lblq ]$ can be upper bounded by $\mathbb{E} \left[ \sum_{j = 0}^{m-1} \sum_{x\in X_1} \mathbb{E} \left[ p_x^{(j)} \right] \right]$. If we can upper bound $\sum_{x\in X_1} \mathbb{E} \left[ p_x^{(j)} \right]$, the problem reduces to upper bounding $\mathbb{E}[m]$. However, the total number of points sampled by the algorithm of \cite{lee2015constructing}, $m$, is only shown to satisfy weak (Markov) concentration which does not imply finiteness of $\mathbb{E} [m]$, let alone an upper bound on the same. It turns out to be quite a challenging task to prove stronger concentration for $m$, since the stopping criterion of the algorithm necessitates understanding the correlations between the potential $\Phi_j^{\text{Id}}$ maintained by the algorithm across iterations $j = 1,\cdots,m$. We resolve this issue by changing the stopping criterion of the algorithm to terminate within $O(d/\gamma^2)$ iterations with probability $1$. More concretely, we stop sampling points as soon as $\sum_{i=0}^{j-1} \frac{4\gamma^2}{\Phi_i^{\text{Id}} (1-4\gamma^2)} + \sum_{i=0}^{j-1} \Phi_i^{\text{Id}} > \frac{8d}{\gamma}$. By the AM-GM inequality, it follows that the stopping criterion must be violated for some $j \le C \nicefrac{d}{\gamma^2}$ for a large enough constant $C$. Indeed this results in the following lemma.

\begin{restatable}{lemma}{lemmaone} \label{lemma:m-ub}
With probability $1$, $m \le \nicefrac{2d}{\gamma^2}$.
\end{restatable}

As a consequence of this lemma, \Cref{alg:RBSS} now necessarily samples fewer points than the algorithm of \cite{lee2015constructing}. In spite of this, we guarantee that it satisfies the $\epsilon$-well-balanced sampling condition.


\subsection{\Cref{alg:RBSS} is an $\epsilon$-well-balanced sampling procedure}

We now address the problem of arguing that in spite of sampling fewer points, the algorithm is still an $\epsilon$-well balanced sampling procedure with reasonable probability. We provide a brief sketch which addresses the first condition in \cref{eq:spec-spar}.

The key result here is to show that the stopping criterion of \Cref{alg:RBSS} guarantees that the number of points sampled by the algorithm \textit{also satisfies $m \ge c \nicefrac{d}{\gamma^2}$} with constant probability, for a sufficiently small constant $c$. This implies that in the final iteration, $u_m$ is also $\Omega (\nicefrac{d}{\gamma^2})$ using the observation that $u_j$ increases by at least a constant in each iteration, which we prove in the Appendix.

\begin{restatable}{lemma}{lemmatwo} \label{lemma:um-lb}
For $\gamma < \nicefrac{1}{4}$ and any $0 \le p < 1$, with probability at least $1-p$, $u_m \geq \nicefrac{p^2d}{8\gamma^2}$.
\end{restatable}

Moreover, by the stopping criterion of the algorithm, we are guaranteed that the gap, $u_m - l_m$ cannot be too large and is $\le \nicefrac{9d}{\gamma}$. This is a nice feature to have, since the matrix $A^T A$ (in \Cref{def:eps-w-b}) has condition number upper bounded by $\frac{u_m}{l_m}$. For sufficiently small $\gamma$, we have that $u_m - l_m \le \nicefrac{9d}{\gamma}$ and $l_m \approx u_m = \Omega (\nicefrac{d}{\gamma^2})$ (from \Cref{lemma:um-lb}), automatically implying that $\nicefrac{u_m}{l_m} = 1 + O(\gamma)$. Thus the normalized eigenvalues of $A^T A$ lie in the interval $\nicefrac{3}{4}$ and $\nicefrac{5}{4}$ for sufficiently small $\gamma$. This results in the following lemma showing that \Cref{alg:RBSS} satisfies the first condition in \Cref{def:eps-w-b} with constant probability.

\begin{lemma}\label{lem:AstartA}
With probability at least $3/4$,
\begin{equation}
    (1 - O(\gamma)) I \preceq A^T A \preceq (1 + O(\gamma)) I.
\end{equation}
\end{lemma}

Following a similar proof as in \cite[]{chen2019active} we also show that the second condition in \cref{eq:eps-wbs-2} for an $\epsilon$-well-balanced sampling procedure is satisfied. Intuitively, since the number of iterations of the algorithm, $m$, is upper bounded now, in principle it is ``easier'' for the algorithm to satisfy $\sum_{i=1}^m \alpha_i = O(1)$ and hence the second condition.

\subsection{Bounding the number of additional labels queries of the algorithm}
Finally, we upper bound the number of times the algorithm queries the labels of points in $X_1$. Define the notation $\lblq$ as the number of additional label queries made by the algorithm. We first introduce some notation that will be useful for the exposition.

\begin{definition}
For a PSD matrix $M$, define the potential function $\Phi_j^{M} = \textsf{Tr} ( M (u_j I - A_j)^{-1} + M (A_j - l_j I)^{-1}))$. 
\end{definition}
We first compute the probability of sampling a point among $X_1$ in each iteration and prove the following lemma.

\begin{lemma} \label{lem:query_complexity}
In each iteration $j = 0,\cdots,m-1$, $\sum_{x \in X_1} p_x^{(j)} = \nicefrac{\Phi_j^D}{\Phi_j^{\text{Id}}}$, where $D = \sum_{x \in X_1} U(x) (U(x))^T$.
\end{lemma}
This implies that the expected number of label queries is upper bounded by:
\begin{equation} \label{eq:lqub}
    \mathbb{E} [\lblq] \le \mathbb{E} \left[ \sum\nolimits_{j=0}^{m-1} \nicefrac{\Phi_j^D}{\Phi_j^{\mathrm{Id}}} \right].
\end{equation}
Next we show that with probability $1$, $\Phi_j^{\text{Id}} \ge \gamma/2$. Plugging into \cref{eq:lqub}, the number of label queries can then be bounded by: $\nicefrac{2}{\gamma} \mathbb{E} [ \sum_{j=0}^{m-1} \Phi_j^D ]$. Observe that the random variable $m-1$ is a stopping time for the martingale sequence $\Phi_0^D,\Phi_1^D,\dots$ Therefore, by Wald's equation we upper bound $\lblq$ by $\nicefrac{2}{\gamma} \mathbb{E} [ \sum_{j=0}^{m-1} \mathbb{E} [\Phi_j^D] ]$. This is further upper bounded using the following lemma:

\begin{restatable}[Bounding the potential]{lemma}{potbound} \label{lemma:Phi-ub}
For any fixed PSD matrix $M \succeq 0$, $\mathbb{E} [\Phi_{j+1}^M] \le \mathbb{E} [\Phi_{j}^M ]$. 
\end{restatable}

The above lemma is proved in \cite{lee2015constructing} for the case of $M=I$. We extend it to the general case by following the same approach. As a result we have the following.
\begin{equation} \label{eq:final-bound}
    \mathbb{E} [\lblq] \le \frac{2}{\gamma} \mathbb{E} \left[ m \mathbb{E} [\Phi_0^D] \right] = \mathbb{E} \left[ m \right] \frac{\gamma \tr (D)}{d}
\end{equation}
The last equation is a short calculation noting the fact that $u_0 = - l_0 = \frac{2d}{\gamma}$ and $A_0 = 0$ in $\Phi_0^D$. Finally, recall that we define $D$ as $\sum_{x \in X_1} U(x) (U(x))^T$. The following lemma relates $D$ to $\rr_X$.


\begin{restatable}[Relating $D$ to $\rr_X$]{lemma}{Drrx}
\label{lemma:trD-vs-rr}
With $D = \sum_{x \in X_1} U(x) (U(x))^T$, $\textsf{Tr} (D) = \rr_X$.
\end{restatable}

Finally, putting together \Cref{lemma:m-ub,lemma:trD-vs-rr} with \cref{eq:final-bound}, we get the bound:
\begin{equation}
    \mathbb{E} [\lblq] \le \frac{2d}{\gamma^2} \frac{\gamma}{d} \rr_X = \frac{2\rr_X}{\gamma}.
\end{equation}
This completes the proof sketch for \Cref{thm:RBSS}.

\begin{remark}
In the case of ridge linear regression, \Cref{alg:main} can be used to construct a weak-coreset for $\beta^*$. The advantage of the algorithm over \cite{pmlr-v108-kacham20a} is that the weak coreset can be constructed without having oracle access to the set of all labels. In particular, the algorithm returns a coreset which uses $O(d \sd/\epsilon)$ bits of memory: the weight matrix $W_S$ requires $\frac{d}{\epsilon}$ bits of memory to store and the subset of $\sd/\epsilon$ points among $X_1$ requires $\frac{d \sd}{\epsilon}$ bits of memory to store. Solving $\argmin_{\beta \in \mathbb{R}^d} \| W_S X \beta - W_S Y \|_2^2$ returns a $(1 + \epsilon)$ approximation to the optimal loss. The memory requirement of our algorithm is of the same order as the algorithm proposed in \cite{pmlr-v108-kacham20a}.
\end{remark}
\section{Lower bound for active ridge regression} \label{sec:lower-bound}
\begin{theorem}\label{thm:lowerb}
For any $\epsilon\leq\nicefrac{1}{100}$, $d$ and any $\lambda \in [1,50]$, there exists a ridge regression instance $(X,Y)$ such that any active learner which outputs $\widehat{\beta}$ satisfying $\|X \widehat{\beta} - Y \|_2^2 + \lambda\| \widehat{\beta} \|_2^2 \leq (1+0.001\epsilon) \argmin_{\beta \in \mathbb{R}^d} \|X \beta - Y\|_2^2 + \lambda\| \beta \|_2^2$ with probability $\ge 3/4$, must query $\Omega (\nicefrac{\sd (X)}{\epsilon})$ labels.
\end{theorem}
The lower bound instance we consider is an adaptation of the construction from \cite{chen2019active} where the authors proved an $O(d/\epsilon)$ sample complexity lower bound for active linear regression ($\lambda=0$ setting). For a large value of $n$, the lower bound is composed of $n$ copies of each standard basis vector $ \frac{1}{\sqrt{n}} e_i$ for each $i = 1,\cdots,d$. The label $Y_x$ of each point $x \in X$ is generated by sampling from a normal distribution with mean $\langle x , \widetilde{\beta} \rangle$ and variance $\nicefrac{(1+\lambda)}{\epsilon}$. Here $\widetilde{\beta}$ itself is a random vector with each coordinate independently chosen as $1 + \lambda$ or $-(1 + \lambda)$.
A simple calculation shows that as $n \to \infty$, the optimal solution to ridge regression almost surely is $\textsf{sgn} (\widetilde{\beta})$ ($\textsf{sgn} (\cdot)$ is applied co-ordinate wise).

In the above instance, the linear regression problem on the overall dataset decomposes into $d$ one-dimensional linear regression problems along each of the standard basis vectors. Moreover, along each dimension, it suffices for the learner to estimate the sign of the co-ordinate of $\widetilde{\beta}$ given access to some number of training examples from a normal distribution with variance $\approx 1/\epsilon$. Any learner requires $O(1/\epsilon)$ labels along each dimension to be able to decide the sign of $\widetilde{\beta}$ with constant probability along that dimension and failing to do so, a learner incurs a constant additive error to \textsf{OPT}. The dependence on the statistical dimension stems from the fact that with the presence of regularization, the value of \textsf{OPT} is larger by a factor of $\approx 1 + \lambda$. Therefore it suffices to sample fewer points (by a factor of $1+\lambda$) to guarantee the same $1+\epsilon$ multiplicative approximation to the loss.

\section{Biases in labels} \label{sec:bias}
\vspace{-0.9em}

In this section, we compare our algorithm with practical active learning algorithms extended to the semi-supervised setting (with a prior labelled dataset). As we discuss below, such algorithms may suffer from slow convergence in the presence of bias in the distribution of labelled points. Bias in labelled datasets is a well documented issue in practice \cite{Chawla_2005,elkan2008,zhu2005semi,SUN2020105306, potamias2012warmstart,geva2019modeling, gururangan2018annotation, shah-etal-2020-predictive}. Previous approaches to tackle this problem propose to ``de-bias'' the dataset or counteract the effect of bias by making distributional assumptions on the biased and unbiased data \cite{Zhang_2020,Chawla_2005,SUN2020105306,elkan2008}. 
In practice, the popular uncertainty sampling algorithm \cite{https://doi.org/10.1002/widm.1132} and the active perceptron algorithm \cite{DASGUPTA20111767} among others are built on the following algorithmic framework: the algorithm maintains a set $D$ of labelled points and a regression / classification function $f$ which are updated over time. In each iteration, the algorithm performs two steps: (I) learn a function $f$ to predict well on the current set of labelled points $D$, and (II) Based on the learned function $f$, carefully choose points in the dataset, query their labels and add them to $D$. Concretely, the uncertainty sampling algorithm proposes to query the labels of points near the decision boundary and adds them to the labelled dataset $D$ in step (II).

A natural extension of this algorithm to the semi-supervised active learning setting is to initialize the dataset $D$ as the set of a-priori labelled points $X_1$. This approach can suffer from poor generalization if the labelled dataset is biased. Indeed, consider the classification example in Figure \ref{fig:1} where the labelled dataset is biased in that the optimal decision boundary on $X_1$ and that on the overall dataset $X_1 \cup X_2$ are not well aligned. To mitigate the effect of this bias, the learner is forced to query the labels of many points in the shaded region in Figure \ref{fig:2}. 
Even worse, the natural uncertainty sampling algorithm which queries labels of points near the decision boundary fails to converge to the optimal classifier because points in the vicinity of the initial decision boundary do not provide any information as to how to improve the classifier. While we discuss these issues in the context of classification for ease of visualization, such examples can be easily seen to exist in the case of regression as well. Here, our algorithm samples provably samples far fewer points, even in the worst case.



\begin{figure*}[t!]
    \centering
    \begin{subfigure}[t]{0.3\textwidth}
        \centering
        \includegraphics[height=3.5cm]{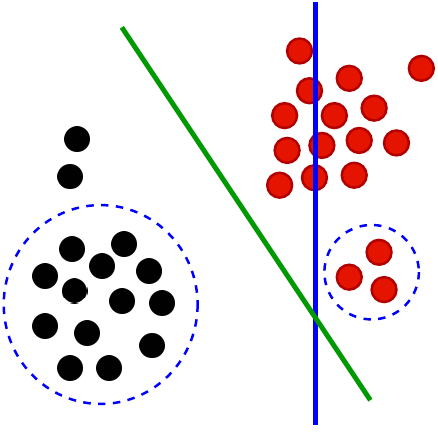}
        \caption{The circled regions denote the set of labelled points which are biased. Blue line is the optimal decision boundary on the labelled dataset, green line is on the entire dataset}
        \label{fig:1}
    \end{subfigure}%
    \quad~ 
    \begin{subfigure}[t]{0.3\textwidth}
        \centering
        \includegraphics[height=3.5cm]{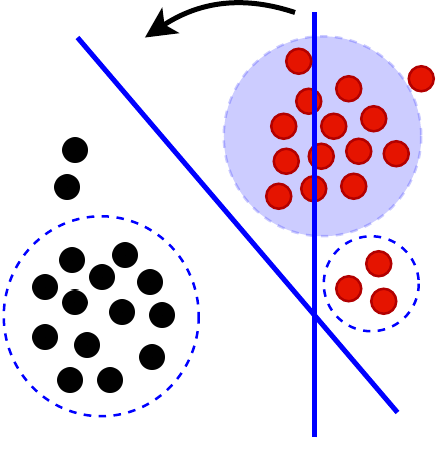}
        \caption{Sampling many points in the shaded blue region eventually pushes the decision boundary towards the optimal one}
        \label{fig:2}
    \end{subfigure}
    \quad~
    \begin{subfigure}[t]{0.3\textwidth}
        \centering
        \includegraphics[height=3.5cm]{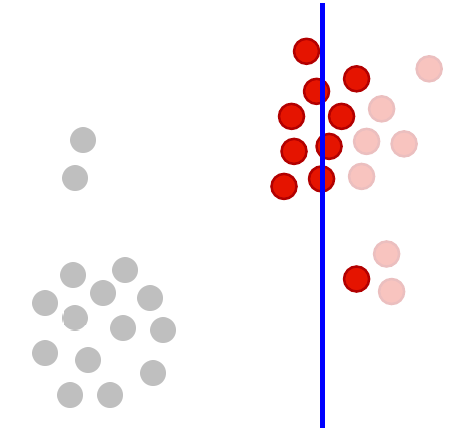}
        \caption{Sampling points near the original decision boundary fails catastrophically - points on either side of the decision boundary have the same labels and do not push it in either direction}
        \label{fig:3}
    \end{subfigure}
\end{figure*}
\vspace{-1em}
\section{Conclusion} \label{sec:conclusion}
\vspace{-0.7em}
In this paper, we introduce the semi-supervised active learning problem and prove a query complexity upper bound of $\nicefrac{\rr_X}{\epsilon}$ in the case of linear regression. In the special case of active ridge regression, this implies a sample complexity upper bound of $O(\nicefrac{\sd}{\epsilon})$ labels which we prove is optimal. The problem also generalizes active kernel ridge regression, where we show a query complexity upper bound of $O(\nicefrac{d_\lambda}{\epsilon})$ improving the results of \cite{alaoui2015fast}. We leave it to future work to prove instance dependent guarantees for semi-supervised active learning under other loss functions such as hinge and log-loss.


{
\small
\bibliographystyle{plainnat}
\bibliography{refs}
}

\section*{Appendix - A}

First we restate and prove \Cref{lemma:trD-vs-rr} which relates the trace of the matrix $D$ to the $\rr_X$ parameter.

\Drrx*

\begin{proof}
Observe that $X = \begin{bmatrix} X_1 \\ X_2 \end{bmatrix} = U \Sigma V^T$. Therefore, $D = U^T S U$ where $S$ is a diagonal matrix with $1$'s on rows corresponding to $x \in X_1$ and $0$'s otherwise. Observe that $X^T S X = V \Sigma U^T S U \Sigma V^T$. Moreover, observe that $X^T S X = \begin{bmatrix} X_1^T & X_2^T \end{bmatrix} S \begin{bmatrix} X_1 \\ X_2 \end{bmatrix} = X_1^T X_1$. Therefore,
\begin{equation}
    D = U^T S U = (V \Sigma)^{-1} X_1^T X_1 (\Sigma V^T)^{-1}
\end{equation}
Tracing both sides and using the commutativity of the trace operator,
\begin{align}
    \tr (D) &= \tr \left( \left( \left( V^T \right)^{-1} \Sigma^{-2} V^{-1} \right) X_1^T X_1 \right) \\
    &= \tr \left( \left( X^T X \right)^{-1} X_1^T X_1 \right)
\end{align}
\end{proof}


\subsection*{Upper bounding the number of points in $X_1$ sampled by \Cref{alg:RBSS} (\Cref{thm:RBSS})}

We first bound the number of points sampled in $X_1$ by \Cref{alg:RBSS}. We begin by re-stating \Cref{lemma:iterations} which explicitly computes the expected number of points sampled sampled by \Cref{alg:RBSS} in terms of various potentials.

\begin{lemma} \label{lemma:iterations}
Recall that \Cref{alg:RBSS} samples a subset of the $n_1+n_2$ points in $X_1 \cup X_2$. The expected number of iterations the algorithm samples a point in $X_1$ is given by: $\mathbb{E} [\lblq] \le \mathbb{E} \left[ \sum_{j=1}^{m-1} \frac{\Phi_j^D}{\Phi_j^{\mathrm{Id}}} \right]$ where $D = \sum_{x \in X_1} U(x) (U(x))^T$.
\end{lemma}
\begin{proof}
Recall from \Cref{eq:m1} that the number of unlabelled points sampled by the algorithm is upper bounded by
\begin{align}
    \lblq &\le \sum_{j=0}^{m-1} \sum_{x\in X_1} p_x^{(j)}\\
    &= \sum_{j=0}^{m-1} \sum_{x\in X_1} \frac{U(x)^T\left((u_jI-A)^{-1}+(A_j-l_jI)^{-1}\right)U(x)}{\Phi_{j}^{\mathrm{Id}}}\\
    &= \sum_{j=0}^{m-1} \frac{\sum_{x\in X_1} \tr \left(U(x)U(x)^T\left((u_jI-A)^{-1}+(A_j-l_jI)^{-1}\right)\right)}{\Phi_{j}^{\mathrm{Id}}}\\
    &= \sum_{j=0}^{m-1} \frac{\tr\left(\left(\sum_{x\in X_1}U(x)U(x)^T\right)\left((u_jI-A)^{-1}+(A_j-l_jI)^{-1}\right)\right)}{\Phi_{j}^{\mathrm{Id}}}\\
    &= \sum_{j=0}^{m-1} \frac{ \Phi_{j}^{\mathrm{D}}}{\Phi_{j}^{\mathrm{Id}}}
\end{align}
where $D = \sum_{x \in X_1} U(x) (U(x))^T$
\end{proof}

\Cref{lemma:iterations} bounds the number of points sampled by \Cref{alg:RBSS} among $X_1$. However the appearance of the potential $\Phi_j^{\mathrm{Id}}$ in the denominator is challenging to bound, so we introduce another result to further upper bound this term.

\begin{lemma} \label{lemma:Phi-lb}
With probability $1$, in every iteration $0 \le j < m$ of the \Cref{alg:RBSS}, $\Phi_j^{\mathrm{Id}} \ge \frac{1}{2} \gamma$.
\end{lemma}
\begin{proof}
Note that for $j \in [m-1]$, $\Phi_j^{\mathrm{Id}} = \textsf{Tr} ((u_j I - A_j)^{-1} + (A_j - l_j I)^{-1})$. Note that $A_j$ is a symmetric matrix. Suppose it is diagonalized as $U \Theta U^T$ where $\Theta = \textsf{diag} (\theta_1,\cdots,\theta_d)$ are its eigenvalues. Then, $\Phi_j^{\mathrm{Id}} = \sum_{t=1}^d \frac{1}{u_j - \theta_t} + \frac{1}{\theta_t - l_j}$. We show in \Cref{lemma:bounded} that $l_j I \preceq A_j \preceq u_j I$. With this constraint on the $\theta_t$'s, by minimizing, we obtain: $\frac{1}{u_j - \theta_t} + \frac{1}{\theta_t - l_j} \ge \frac{4}{u_j-l_j}$. Therefore, $\Phi_j^{\mathrm{Id}} \ge \frac{4d}{u_j - l_j}$. Furthermore, by the stopping criterion of the algorithm, in every iteration $j < m$ of the algorithm, $u_j - l_j \le \frac{8d}{\gamma}$. Therefore, for every $j = 0,1,\cdots,m-1$, $\Phi_j^{\mathrm{Id}} \ge \frac{1}{2} \gamma$.
\end{proof}

Using \Cref{lemma:iterations,lemma:Phi-lb}, we can bound the query complexity by
\begin{align} \label{eq:2ndlast}
    \mathbb{E} [\lblq] \leq \frac{2}{\gamma}\mathbb{E} \left[ \sum_{j=0}^{m-1} \Phi_j^D \right]
\end{align}
In order to bound $\mathbb{E} \left[ \sum_{j=0}^{m-1} \Phi_j^D \right]$, we use \Cref{lemma:Phi-ub}. We re-state it here for convenience:

\potbound*

\begin{proof}
Recall that $\Phi_j^M = \textsf{Tr} ( M(u_j I - A_j)^{-1} + \textsf{Tr} ( M(A_j - l_j I)^{-1})$. $\Phi^M_{j+1}$ can be written as $\textsf{Tr} ( M(u_{j+1} I - A_j - w_j w_j^T)^{-1)} + \textsf{Tr} ( M (A_j + w_j w_j^T - l_{j+1} I)^{-1})$. Following a similar approach as BSS Lemma 3.3 and 3.4, invoking the Sherman-Morrison inversion formula,
\begin{equation}
    \left(u_{j+1} I - A_j - w_j w_j^T \right)^{-1} = (u_{j+1} I - A_j)^{-1} + \frac{(u_{j+1} I - A_j)^{-1} w_j w_j^T (u_{j+1} I - A_j)^{-1}}{1 - w_j^T (u_{j+1} I - A_j)^{-1} w_j}.
\end{equation}
Multiplying by $M$ and tracing both sides,
\begin{equation}
   \textsf{Tr} \left(M \left(u_{j+1} I - A_j - w_j w_j^T \right)^{-1} \right) = \textsf{Tr} (M (u_{j+1} I - A_j)^{-1}) + \frac{\textsf{Tr} (M (u_{j+1} I - A_j)^{-1} w_j w_j^T (u_{j+1} I - A_j)^{-1})}{1 - w_j^T (u_{j+1} I - A_j)^{-1} w_j}.
\end{equation}
Note that with probability $1$, $w_j w_j^T \preceq \gamma (u_j I - A_j) \preceq \gamma (u_{j+1} I - A_j)$. Therefore, $w_j^T (u_{j+1} - A_j)^{-1} w_j \le \gamma$. Therefore,
\begin{equation}
    \textsf{Tr} \left(M\left(u_{j+1} I - A_j - w_j w_j^T \right)^{-1} \right) \le \textsf{Tr} (M(u_{j+1} I - A_j)^{-1}) + \frac{\textsf{Tr} (M(u_{j+1} I - A_j)^{-1} w_j w_j^T (u_{j+1} I - A_j)^{-1})}{1 - \gamma}.
\end{equation}
Finally, using linearity of expectation and noting that $\mathbb{E} \left[w_j w_j^T \middle| A_j \right] = \frac{\gamma}{\Phi_j^{\mathrm{Id}}} I$, we have that,
\begin{align}
    \mathbb{E} \left[ \textsf{Tr} \left( M\left(u_{j+1} I - A_j - w_j w_j^T \right)^{-1} \right) \right] &\le \mathbb{E} \left[ \textsf{Tr} (M(u_{j+1} I - A_j)^{-1}) \right] + \mathbb{E} \left[ \frac{\gamma}{\Phi_j^{\mathrm{Id}} (1 - \gamma)} \textsf{Tr} (M(u_{j+1} I - A_j)^{-2}) \right]. \label{eq:13}
\end{align}
By a similar calculation as before,
\begin{align} \label{eq:14}
    \mathbb{E} \left[ \textsf{Tr} \left(M\left(A_j + w_j w_j^T - l_{j+1} I  \right)^{-1} \right) \right] &\le \mathbb{E} \left[ \textsf{Tr} (M(A_j - l_{j+1} I )^{-1}) \right] - \mathbb{E} \left[ \frac{\gamma}{\Phi_j^{\mathrm{Id}} (1+2\gamma)} \textsf{Tr} \left(M(A_j - l_{j+1} I )^{-2} \right) \right].
\end{align}
Note the difference from before, for $\gamma \le \frac{1}{4}$, we use the inequality $w_j w_j^T \preceq 2\gamma (A_j - l_{j+1} I)$ which we derive in \Cref{lemma:wjwjT-l}. This appears as the $1+2\gamma$ factor in the denominator of the second term in \cref{eq:14}.

Now observe that, $u_{j+1} - u_j = \frac{\gamma}{\Phi_j^{\mathrm{Id}} (1-2\gamma)} \ge \frac{\gamma}{\Phi_j^{\mathrm{Id}} (1-\gamma)}$ and $l_{j+1} - l_j = \frac{\gamma}{\Phi_j^{\mathrm{Id}} (1+2\gamma)}$. Therefore, adding \cref{eq:13} and \cref{eq:14} together,
\begin{align}
    \mathbb{E} \left[ \Phi^M_j \right] \le& \mathbb{E} \left[ \textsf{Tr} (M(u_{j+1} I - A_j )^{-1} + M(A_j - l_{j+1} I )^{-1}) \right] \nonumber\\
    &+ \mathbb{E} \left[ (u_{j+1} - u_j) \textsf{Tr} (M(u_{j+1} I - A_j)^{-2}) - (l_{j+1} - l_j) \textsf{Tr} (M(A_j - l_{j+1} I)^{-2}) \right] \label{eq:16}
\end{align}
Define $\Delta_u = u_{j+1} - u_j$ and $\Delta_l = l_{l+1} - l_j$ and for $t \in [0,1]$, define the function
\begin{equation}
    f (t) = \textsf{Tr} \Big( M( (u_j + \Delta_u t) I - A_j)^{-1} + M(A_j - (l_j + t \Delta_l) I)^{-1} \Big).
\end{equation}
 Under the assumption $l_j I \preceq A_j \preceq u_j I$, the function $f(t)$ is convex in $t$. Therefore, $f(0) - f(1) \ge - \frac{\mathrm{d} f(t)}{\mathrm{d}t} \Big|_{t=1}$. In \cref{eq:16} observe that the RHS is precisely $f(1) - \frac{\mathrm{d} f(t)}{\mathrm{d}t} \Big|_{t=1}$. Upper bounding this by $f(0)$, results in the equation
 \begin{equation}
     \mathbb{E} [\Phi_{j+1}^M] \le \mathbb{E} \left[ \textsf{Tr} (M(u_j I - A_j )^{-1} + M(A_j - l_j I )^{-1}) \right] = \mathbb{E} [\Phi_j^M].
 \end{equation}
\end{proof}

As a corollary of \Cref{lemma:Phi-ub}, we have the following result.

\begin{corollary} \label{corr:1}
For any fixed PSD matrix, $M$, $\mathbb{E} [ \sum_{j=0}^{m-1} \Phi_j^M] \le \left(\frac{1}{u_0} - \frac{1}{l_0} \right) \mathbb{E} [m] \textsf{Tr} (M)$.
\end{corollary}
\begin{proof}
From Wald's equation,
\begin{align}
    \mathbb{E} \left[ \sum_{j=0}^{m-1} \Phi_j^M \right] &= \mathbb{E} \left[ \sum_{j=0}^{m-1} \mathbb{E} \left[ \Phi_j^M \right] \right] \\
    &\overset{(i)}{\le} \mathbb{E} \left[ \sum_{j=0}^{m-1} \mathbb{E} \left[ \Phi_0^M \right] \right] \\
    &= \mathbb{E} \left[ m \right] \mathbb{E} \left[ \Phi_0^M \right] \\
    &= \mathbb{E} \left[ m \right] \textsf{Tr} (M (u_0I)^{-1} + M (-l_0 I)^{-1}) \\
    &= \left(\frac{1}{u_0} - \frac{1}{l_0} \right) \mathbb{E} [m] \textsf{Tr} (M)
\end{align}
where $(i)$ follows from \Cref{lemma:Phi-ub}.
\end{proof}

Finally, invoking \cref{eq:2ndlast} and \Cref{corr:1} with the choice of $M = D$,
\begin{equation} \label{eq:last}
    \mathbb{E} [\lblq] \le \frac{2}{\gamma} \mathbb{E} \left[ \sum_{j=0}^{m-1} \Phi_j^D \right] \le \frac{2}{\gamma} \left( \frac{1}{u_0} - \frac{1}{l_0} \right) \mathbb{E} [m] \textsf{Tr} (D) \le \frac{2}{d} \rr_X \mathbb{E} [m]
\end{equation}
where the last inequality uses the fact that $u_0 = \frac{2d}{\gamma}$ and $l_0 = -\frac{2d}{\gamma}$ and $\tr (D) = \rr_X$ from \Cref{lemma:trD-vs-rr}. The final quantity to bound is $\mathbb{E} [m]$ which we carry out using \Cref{lemma:m-ub}.

\lemmaone*

\begin{proof}
Assuming that the algorithm has not terminated till the $(t+1)^{th}$ iteration, $u_t - l_t = u_0 - l_0 + \sum_{j=0}^{t-1} \frac{4\gamma^2}{\Phi_j^{\mathrm{Id}} (1-4\gamma^2)} < \frac{8d}{\gamma}$ (this uses the fact that $u_{j+1} - u_j = \frac{\gamma}{\Phi_j^{\mathrm{Id}} (1-2\gamma)}$ and $l_{j+1} - l_j = \frac{\gamma}{\Phi_j^{\mathrm{Id}} (1+2\gamma)}$). Observe that the event,
\begin{align}
    \{ m \ge t \} &= \{ u_t - l_t < 8d/\gamma \} = \left\{ \sum_{j=0}^{t-1} \frac{4\gamma^2}{\Phi_j(1-4\gamma^2)} + \sum_{j=0}^{t-1}\Phi_j^{\mathrm{Id}} < \frac{4d}{\gamma} \right\} \overset{(ii)}{\subseteq} \left\{ 2\gamma \cdot t < \frac{4d}{\gamma} \right\}
\end{align}
where the last equation uses the fact that $u_{j+1} - l_{j+1} = u_j - l_j + \frac{\gamma}{\Phi_{j} (1-2\gamma)} - \frac{\gamma}{\Phi_{j} (1+2\gamma)}$ with $u_0 - l_0 = \frac{4d}{\gamma}$, and the  $(ii)$ uses the AM$\ge$GM inequality. Therefore, with $t = \frac{2d}{\gamma^2}$, the event $\{ m \ge t \}$ happens with probability $0$.
\end{proof}

As a result of \cref{eq:last} and \Cref{lemma:m-ub} we have that,
\begin{equation}
    \mathbb{E} [\lblq] \le \frac{2d}{\gamma^2} \frac{2}{d} \rr_X = \frac{4 \rr_X}{\gamma^2}.
\end{equation}
This completes the bound on the number of points sampled by \Cref{alg:RBSS}. Next we move on to showing that \Cref{alg:RBSS} is indeed an $\epsilon$-well-balanced sampling procedure which will complete the proof of \Cref{thm:RBSS}.

\subsection*{\Cref{alg:RBSS} is $\epsilon$-well balanced sampling procedure}

In order to satisfy the first property of \Cref{def:eps-w-b}, we need to show that $A^TA$ is well conditioned and that its normalized eigenvalues lie in an interval $[1-\gamma, 1+\gamma]$. We discuss later that $A^TA=\frac{1}{\nicefrac{(u_m+l_m)}{2}}A_m$, where $m$ is the number of iterations the while loop in \Cref{alg:RBSS} runs for, and $A_j$ is as defined in \Cref{alg:RBSS}. 
As we show in \Cref{lemma:bounded}, the eigenvalues of $A_j$ for any $j$ are bounded between $u_j$ and $l_j$ and when the algorithm terminates, the gap between $u_m$ and $l_m$ is $O (d / \gamma)$. Moreover, as we show later in \Cref{lemma:um-lb}, with constant probability, $u_m$ is also lower bounded by $\Omega(d/\gamma^2)$. By construction, this will show that $l_m \approx u_m = \Omega (d/\gamma^2)$ and $u_m - l_m = O (d/\gamma)$. These two conditions show that the eigenvalues of $A^T A = \frac{1}{\nicefrac{(u_m+l_m)}{2}} A_m$ lie in the interval $[1-\gamma,1+\gamma] \subseteq [3/4,5/4]$ (for sufficiently small $\gamma$) showing indeed that the first property for $\epsilon$-well balanced sampling procedures is satisfied by \Cref{alg:RBSS}. First we show the key result of this section that with constant probability $u_m$ is indeed lower bounded by $\Omega (\nicefrac{d}{\gamma^2})$.

\lemmatwo*

\begin{proof}
First, observe that $u_m > \sum_{j=0}^{m-1}\frac{\gamma}{\Phi_j^{\mathrm{Id}}}$ and $\left( \sum_{j=0}^{t-1} \frac{1}{\Phi_j} \right) \left( \sum_{j=0}^{t-1} \Phi_j \right) \ge t^2$, we want to analyze
\begin{align}
    \mathrm{Pr}\left( \frac{p^2 d}{8 \gamma^3} \leq \frac{m^2}{\sum_{j=0}^{m-1} \Phi_j^{\mathrm{Id}}}\right) &= \mathrm{Pr}\left(\sum_{j=0}^{m-1} \Phi_j^{\mathrm{Id}} \cdot \frac{p^2 d}{8 \gamma^3} \leq m^2\right)\\
    &\overset{(i)}\geq \mathrm{Pr}\left(\frac{p^2 d^2}{\gamma^4} \leq m^2\right)
\end{align}
Where the last sufficient condition $(i)$ comes from the stopping criterion, which implies $\sum_{j=0}^{m-1} \Phi_j^{\mathrm{Id}} \leq \frac{8d}{\gamma}$.

\noindent Now we prove an upper bound to $\mathrm{Pr}(m < g)$ where $g \triangleq \frac{pd}{\gamma^2}$.
\begin{align*}
    \mathrm{Pr}\left(m < g \right) &\overset{(i)}= \mathrm{Pr}\left( \sum_{j=0}^{g-1} \frac{\gamma^2}{\Phi_j^{\mathrm{Id}} (1 - 4\gamma^2)} + \frac{4d}{\gamma} + \sum_{j=0}^{g-1} \Phi_j^{\mathrm{Id}} \geq \frac{8d}{\gamma}  \right)\\
    &\overset{(ii)}\leq \frac{\mathbb{E}\left[\sum_{j=0}^{g-1}\frac{\gamma^2}{\Phi_j^{\mathrm{Id}} (1 - 4\gamma^2)} + \sum_{j=0}^{g-1} \Phi_j^{\mathrm{Id}}\right]}{4d/\gamma}\\
    &\overset{(iii)}\leq \frac{g\mathbb{E}[\Phi_0^{\mathrm{Id}}] + \frac{\gamma^2}{1 - 4\gamma^2} \mathbb{E}\left[ \sum_{j=0}^{g-1}\frac{1}{\Phi_j^{\mathrm{Id}}} \right]}{4d/\gamma}\\
    &\overset{(iv)}\leq \frac{g\gamma + \frac{2 \gamma g}{1 - 4\gamma^2}}{4d/\gamma}
\end{align*}
where $(i)$ comes form the stopping criterion and the update rules for $u_j$ and $l_j$, $(ii)$ is Markov's inequality, $(iii)$ follows by \Cref{lemma:Phi-ub}, and $(iv)$ from \Cref{lemma:Phi-lb}. Using the fact that $\gamma < \frac{1}{4}$,
\begin{align}
    \mathrm{Pr}\left(m < g \right) &\le \frac{g\gamma^2}{d} = p
\end{align}
\end{proof}

Next we define the ``good'' event $\Gamma$ that $u_m$ is indeed $\Omega (d/\gamma^2)$. Note that $\Gamma$ occurs with constant probability using \Cref{lemma:um-lb}.

\begin{definition} \label{def:Gamma}
Define $\Gamma$ as the event that $\{ u_m \ge \frac{d}{64\gamma^2}\}$. From \Cref{lemma:um-lb}, $\mathrm{Pr} (\Gamma) \ge \frac{3}{4}$.
\end{definition}

From the stopping criterion of the algorithm, we know that $u_j-l_j\leq \nicefrac{8d}{\gamma}$, for $j<m$. We show that even for $j=m$ this inequality is true with a larger choice of constant.

\begin{lemma} \label{lemma:umlm-ub}
For $\gamma < 1$, $u_m - l_m \leq \nicefrac{9d}{\gamma}$.
\end{lemma}
\begin{proof}
From \Cref{lemma:Phi-lb}, we know $\phi_{m-1}^{\mathrm{Id}} \geq \gamma/2$. Which implies $\gamma/\phi_{m-1}^{\mathrm{Id}} \leq 2$. By the stopping criterion of the algorithm, $u_{m-1} - l_{m-1} < 8d/\gamma$. Using these two,
\begin{align*}
    u_m-l_m &= u_{m-1}-l_{m-1} + \frac{\gamma}{\phi_{m-1}^{\mathrm{Id}}}\left(\frac{1}{1-2\gamma} - \frac{1}{1+2\gamma}\right)\\
    &\leq 8d/\gamma + 2\left(\frac{1}{1-2\gamma} - \frac{1}{1+2\gamma}\right)\\
    &\leq 9d/\gamma
\end{align*}
\end{proof}


Next we show that under the event $\Gamma$, the matrix $A_m$ is PSD, which is crucial towards bounding its condition number.

\begin{lemma}
For $\gamma \leq \frac{1}{300}$, if the event $\Gamma$ (defined in \Cref{def:Gamma}) occurs, $l_m > 0$.
\end{lemma}
\begin{proof}
First observe that,
\begin{align}
    u_m &= u_0 + \frac{\gamma}{(1-2\gamma)}\sum_{j=0}^{m-1}\frac{1}{\Phi_j^{\mathrm{Id}}}\\
    \implies &\sum_{j=0}^{m-1}\frac{1}{\Phi_j^{\mathrm{Id}}} = \left(u_m-\frac{2d}{\gamma}\right) \left( \frac{1-2\gamma}{\gamma} \right)\\
    \implies &l_m = \frac{-2d}{\gamma} + \frac{\gamma}{1+2\gamma} \left(u_m-\frac{2d}{\gamma}\right) \left( \frac{1-2\gamma}{\gamma} \right)
\end{align}
Thus if $u_m$ is large enough, the RHS will be $> 0$. It suffices to assume $\gamma
\leq \frac{1}{300}$ for this statement to be true since conditioned on $\Gamma$, $u_m \ge \frac{d}{64 \gamma^2}$.
\end{proof}

Finally, conditioned on the event $\Gamma$ and invoking \Cref{lemma:umlm-ub}, we bound the condition number of $A_m$.

\begin{lemma} \label{lemma:3456}
Conditioned on the event $\Gamma$ (defined in \Cref{def:Gamma}), Algorithm \ref{alg:RBSS}'s last iteration matrix $A_m$ has condition number $\frac{\lambda_{\max} (A_m)}{\lambda_{\min} (A_m)} \le \frac{u_m}{l_m} \le 1+3456\gamma$, for $\gamma \leq \frac{1}{7 00}$.
\end{lemma}
\begin{proof}
From \Cref{lemma:bounded}, the condition number of $A_m$ is at most,
\begin{align}
    \frac{u_m}{l_m} = \left(1 - \frac{u_m-l_m}{u_m}\right)^{-1}
\end{align}
Hence, it suffices to prove that $(u_m-l_m)/u_m$ is $\leq c\gamma$ with constant probability, assuming that $c\gamma\leq \frac{5}{6}$. We know from \Cref{lemma:umlm-ub}, $u_m-l_m \leq  \frac{9d}{\gamma}$. Hence, it suffices to show that under the event $\Gamma$,
\begin{align}
    \frac{9d/\gamma}{u_k} \leq c\gamma
    \iff \frac{9d}{c\gamma^2} \leq u_m
\end{align}
Conditioned on the event $\Gamma$, $u_m \ge \frac{d}{64\gamma^2}$. Therefore, it suffices to choose $c \ge 576$. As $\gamma \leq \frac{1}{700}, c\gamma \leq \frac{5}{6}$. Finally,
\begin{equation}
    \left(1 - \frac{u_m-l_m}{u_m}\right)^{-1} \le 1 + \frac{c \gamma}{1 - c \gamma} \le 1 + 3456 \gamma.
\end{equation}
\end{proof}

\Cref{lemma:3456} directly translates to an upper bound on the eigenvalues of $A^TA$ which are nothing but the eigenvalues of $A_m$ up to a scaling factor of $\nicefrac{(u_m+l_m)}{2}$.

\begin{lemma}\label{lem:AstartA}
Conditioned on the event $\Gamma$ (defined in \Cref{def:Gamma}), $(1 - 1728 \gamma) I \preceq A^T A \preceq (1 + 1728 \gamma) I$.
\end{lemma}
\begin{proof}
From \Cref{lemma:3456}, $\frac{u_m}{l_m} \le 1 + 3456 \gamma$.
\begin{align*}
    A^T A = \frac{1}{\midd}\sum_{j=1}^{m} w'_j U (x_j) U(x_j)^T = \frac{A_m}{\midd}
\end{align*}
Therefore, $\lambda(A^T A) \in \left[\frac{l_m}{\midd},\frac{u_m}{\midd}\right]$. Also, given $\frac{u_m}{l_m} \leq 1 + 3456\gamma$, we have
\begin{align}
    \frac{u_m+l_m}{l_m} &\leq 2+3456\gamma\\
    \implies \frac{2}{2 + 3456\gamma} &\leq \frac{l_m}{\frac{u_m+l_m}{2}} \\
    \implies 1 - 1728\gamma &\leq \frac{l_m}{\midd}
\end{align}
A similar approach can be used to prove that $\frac{u_m}{\midd} \leq 1 + 1728\gamma$.
\end{proof}

This completes the proof in showing that \Cref{alg:RBSS} satisfies the first property of being an $\epsilon$-well-balanced sampling procedure. Next, we prove that \Cref{alg:RBSS} satisfies the second property ($\sum_{j=0}^{m-1} \alpha_j = O(1)$ and $\alpha_j K_{D_j} = O(\epsilon)$) which will complete the proof of \Cref{thm:RBSS} which we restate below.

\thmRBSS*

\begin{proof}
To show that \Cref{alg:RBSS} is an $\epsilon$-well-balanced sampling procedure, recall from the definition that we must show that with probability $\ge \frac{3}{4}$,
\begin{enumerate}
    \item $\frac{3}{4} I \preceq A^T A \preceq \frac{5}{4} I$
    \item $\sum_{j=0}^{m-1} \alpha_j = O(1)$ and for all $j=0,\cdots,m-1$, $\alpha_j K_{D_j} = O(\epsilon)$.
\end{enumerate}
Conditioned on the event $\Gamma$ which holds with probability $\ge \frac{3}{4}$, we show that both of these properties hold. 

From \Cref{lem:AstartA}, we have that $(1 - 1728 \gamma) I \preceq A^T A \preceq (1 + 1728 \gamma) I$. With $\gamma = \sqrt{\epsilon}/C_0$ with $\epsilon < 1$ and sufficiently large $C_0 > 0$, this implies that $\frac{3}{4} I \preceq A^T A \preceq \frac{5}{4} I$ which proves the first part.

On the other hand, to bound $\sum_{j=0}^{m-1} \alpha_j$, observe that
\begin{align*}
    \sum_{j=0}^{m-1}\alpha_j &=  \sum_{j=0}^{m-1}\frac{\gamma}{\phi_j^{\mathrm{Id}}}\cdot    \frac{1}{\midd} \leq \sum_{j=0}^{m-1}\frac{4}{\frac{u_m+l_m}{2}} \leq \frac{2d}{\gamma^2}\frac{8}{u_m} \leq 1024
\end{align*}
where we use the fact that $u_m \ge \frac{d}{64 \gamma^2}$ conditioned on $\Gamma$. Following the proof of \citet[Lemma 5.1]{chen2019active}, we bound $\alpha_j K_{D_j}$ as follows:
\begin{align*}
    \alpha_jK_{D_j} = \frac{\gamma}{\midd}\cdot \frac{u_j-l_j}{2} = \gamma\frac{u_j-l_j}{u_m+l_m} \leq 512\gamma^2 = \frac{512 \epsilon}{C_0^2}
\end{align*}
where we upper bound $u_j-l_j \leq \frac{8d}{\gamma}$ using the stopping criterion of \Cref{alg:RBSS}, and lower bound $u_m+l_m\geq u_m \geq d/64\gamma^2$ conditioned on the event $\Gamma$. We also substitute $\gamma = \frac{\sqrt{\epsilon}}{C_0}$ and choose $C_0$ appropriately.
\end{proof}

\subsection*{Proof of lower bound (\Cref{thm:lowerb})}
\begin{theorem}\label{thm:lowerb}
For any $d$, $\epsilon\leq\nicefrac{1}{100}$ and any $\lambda \in [1,50]$, there exists $(X,y)$ such that any algorithm which outputs $\widehat{\beta}$ that satisfies $\|X\widehat{\beta}-y\|_2^2+\lambda\|\widehat{\beta}\|_2^2\leq (1+0.001\epsilon) \argmin_\beta \|X \beta - y\|_2^2+\lambda\|\beta\|_2^2$ with probability $\geq 3/4$, queries $m = \Omega( \nicefrac{\sd (X)}{\epsilon})$ labels.
\end{theorem}

The proof follows similar to \cite{chen2019active} lower bound proof. Consider $nd\times d$ matrix $X = \nicefrac{1}{\sqrt{n}}\widehat{X}$ where $\widehat{X}$ consists $n$ copies of the standard basis vector $e_i$ for each $i=1,\cdots,d$. Let $y$ value be $\frac{1}{\sqrt{n}} \left(\widehat{X}\widehat{\beta} + \mathcal{N}(0,\nicefrac{(1+\lambda)}{\epsilon})\right)$, where $\widehat{\beta}=\pm(1+\lambda)$ (where $\pm$ sign chosen at random), and $\mathcal{N}(0,\nicefrac{(1+\lambda)}{\epsilon})$ is Gaussian noise. Let $n$ be very large tending to $\infty$.
In this case, the optimal solution $\beta^{*}$ of ridge regression is:
\begin{align}
    \beta^{*} &= \left(X^TX+\lambda I\right)^{-1}X^Ty\\
    &= \frac{\widehat{X}^T\widehat{w}}{n(1+\lambda)}\\
    &= \sign(\widehat{\beta})
\end{align}
where we get the second last equation because a $0$-mean Gaussian noise averaged over a large number of samples converge to $0$. The optimal loss becomes:
\begin{align}
    \opt &= \left\|X \beta^* - \frac{1}{\sqrt{n}}(\widehat{X}\widehat{\beta}+\mathcal{N}(0,\nicefrac{(1+\lambda)}{\epsilon}))\right\|_2^2 + \lambda\| \beta^* \|_2^2\\
    &= d\left(\lambda^2+\frac{1+\lambda}{\epsilon}\right) + d\lambda\\
    &= d\left(\lambda(1+\lambda) + \frac{1+\lambda}{\epsilon}\right)
\end{align}
\begin{claim}\label{claim:001}
There exists a subset $\mathcal{M}=\{\beta_1,\beta_2,\cdots,\beta_n \} \subset \mathbb{R}^{d}$ such that:
\begin{enumerate}
    \item $\|X \beta\|_2^2 \leq d$ for all $\beta \in \mathcal{M}$
    \item $\|\beta\|_\infty \leq 1$ for all $\beta \in \mathcal{M}$
    \item $\|X \beta - X \beta'\|_2^2 \ge 0.002d\left(\epsilon\lambda(1+\lambda) + 1+\lambda\right)$ for distinct $\beta,\beta'$ in $\mathcal{M}$
    \item $n\geq 2^{(1 - 0.011 (1+\lambda)) d-1}$
\end{enumerate}
\end{claim}
\begin{proof}
We follow the proof of \cite[Claim 8.2]{chen2019active} to construct $\mathcal{M}$ as a packing set from the set of $2^d$ vectors, $U = \{\pm 1 \}^d$ in Procedure \textsc{ConstructM} as defined in \cite[Algorithm 4]{chen2019active} which we restate in \Cref{alg:constructm}. In each iteration, \textsc{ConstructM} removes at most $k = \sum_{i=0}^{0.002d (\epsilon \lambda(1 + \lambda) + 1 + \lambda)/4} \binom{d}{i}$ points from $U$. Note that by the assumption $\lambda \le 50$ and $\epsilon \le \frac{1}{100}$, we have that $0.002d (\epsilon \lambda(1 + \lambda) + 1 + \lambda)/4 \le 0.001d ( 1 + \lambda) \le 0.2 d$.

Simplifying, $k \le \binom{d}{ 0.001 d (1+\lambda)} \frac{d-0.2d+1}{d-0.4d+1} \le 2 \left( \frac{e}{0.001 (1 + \lambda)} \right)^{0.001 d (1 + \lambda)} \le 2^{0.011 d (1 + \lambda)}$. When the algorithm terminates, $n \geq 2^d / k$ which proves point (4).
\end{proof}

\begin{algorithm}[htb]
	\caption{\textsc{\textsc{ConstructM}}}
	\label{alg:constructm}
	\begin{algorithmic}[1]
	    \State \textbf{Input:} Dimension parameter $d$
		\State Define $n = 0$ and $U = \{ \pm 1 \}^d$.
		\While{$U \ne 0$}
		\State Choose any $\beta \in U$ and remove all $\beta' \in U$ from $U$ such that $\| X \beta' - X \beta \|_2^2 \le 0.002d (\epsilon \lambda (1 + \lambda) + 1 + \lambda)$.
		\State $n \gets n + 1$; $\mathcal{M} \gets \mathcal{M} \cup \{ \beta \}$
		\EndWhile
		\State Return $\mathcal{M}$.
	\end{algorithmic}
\end{algorithm}

\begin{proof}[Proof of \Cref{thm:lowerb}]
We want $\widehat{\beta}$ such that 
\begin{align}
    \|X\widehat{\beta} - Y\|_2^2 + \lambda\|\widehat{\beta}\|_2^2 \leq (1+0.001\epsilon)\opt \\
    \implies \|X\widehat{\beta} - X\beta^* \|_2^2 + \lambda\|\widehat{\beta} - \beta^* \|_2^2 \leq (0.001\epsilon) \opt\\
    \implies \| X\widehat{\beta} - X \beta^*\|_2^2 \leq (0.001\epsilon) \opt \label{eq:03}
\end{align}
From proof of lower bound in \cite{chen2019active}[Theorem 8.1], denoting $x_1,\cdots,x_m$ as the points sampled by the learner, we get that any algorithm which finds $\widehat{\beta}$ satisfying \Cref{eq:03} with probability $\geq 3/4$ needs at least $m$ samples such that:
\begin{align}
    \log (n) - 1 - \frac{1}{4} \log (n) \le \sum_{i=1}^m \frac{1}{2} \log \left( 1 + \frac{1}{\nicefrac{1}{\epsilon (1 + \lambda)}} \right) \le \frac{1}{2} m \epsilon ( 1 + \lambda).
\end{align}
Plugging in the bound on $n$ from Claim 1, we get that,
\begin{align}
    m \geq \frac{1.4 (1 - 0.011(1+\lambda))d}{\epsilon (1 + \lambda)} \gtrsim \frac{\sd (X)}{\epsilon}
\end{align}
Where we use the fact that $\lambda\in [1,50]$ and $\sd = \frac{d}{(1 + \lambda)}$ for the considered input $X$.
\end{proof}

\section*{Appendix - B}

\begin{lemma} \label{lemma:wjwjT-l}
For $\gamma \le \frac{1}{4}$, $w_j w_j^T \preceq 2\gamma (A_j - l_{j+1} I)$.
\end{lemma}
\begin{proof}
First observe that,
\begin{equation} \label{eq:15}
    w_j w_j^T \preceq \gamma (A_j - l_j I) = \gamma (A_j - l_{j+1} I) + \gamma (l_{j+1} - l_j) I
\end{equation}
Therefore, it suffices to show that $l_{j+1} - l_j \preceq (A_j - l_{j+1} I)$, or in other words, $l_{j+1} - l_j \le \lambda_{\min} (A_j - l_{j+1} I)$ to complete the proof. By definition,
\begin{align}
    l_{j+1} - l_j = \frac{\gamma}{(1-2\gamma)\Phi_j^{\mathrm{Id}}} \le \frac{\gamma}{1-2\gamma} \lambda_{\min} (A_j - l_j I) \le \frac{1}{2} \lambda_{\min} (A_j - l_j I)
\end{align}
where the last inequality uses the fact that $\gamma \le \frac{1}{4}$. Therefore, $2 l_{j+1} - l_j \le \lambda_{\min} (A_j)$ and $l_{j+1} - l_j \le \lambda_{\min} (A_j - l_{j+1})$. Plugging this back into \cref{eq:15},  we arrive at the claim of the lemma.
\end{proof}

\begin{lemma} \label{lemma:bounded}
For $\gamma < 1$, in each iteration $j = 0,\cdots,m$ of \Cref{alg:RBSS}, the condition $l_j I \preceq A_j \preceq u_j I$ is satisfied.
\end{lemma}
\begin{proof}
The proof follows by induction. For $j = 0$, $A_j = 0$ and trivially satisfies the condition $\frac{2d}{\gamma} I = -l_j I \preceq A_j \preceq u_j I = \frac{2d}{\gamma} I$. By the induction hypothesis, we assume that $l_j I \preceq A_j \preceq u_j I$ henceforth in the proof.  For any point $x$, observe that,
\begin{align}
    p_x \Phi_j^{\text{Id}} &= U(x)^T \left( (u_j I - A_j)^{-1} + (A_j - l_j I)^{-1} \right) U(x) \label{eq:1230}\\
    &\ge U(x)^T (u_j I - A_j)^{-1} U(x) \label{eq:1231}
\end{align}
Observe that for any vector $v$ and PSD matrix $B$, $vv^T \preceq (v^T B^{-1} v) B$.
Therefore, for any point $x$,
\begin{align}
    U(x) U(x)^T &\preceq (U(x)^T (u_j I - A_j)^{-1} U(x) ) (u_j I - A_j) \\
    &\overset{(i)}{\preceq} p_x \Phi_j^{\text{Id}} (u_j I - A_j) \label{eq:1232}
\end{align}
where $(i)$ uses \cref{eq:1231}.
Similarly by lower-bounding \cref{eq:1230} by $U(x)^T (A_j - l_j I)^{-1} U(x)$ and use a similar approach to prove that for any $x$,
\begin{equation} \label{eq:lA-lb}
    U(x) U(x)^T \preceq p_x \Phi_j^{\text{Id}} (A_j  - l_j I)
\end{equation}
Choosing $x = x_j$ in \cref{eq:1232}, as a special case, 
\begin{equation}
    A_{j+1} - A_j = \frac{\gamma}{p_j \Phi_j^{\text{Id}}} U(x_j) U(x_j)^T \preceq \gamma (u_j I - A_j) \label{eq:uA-ub}
\end{equation}
Using the induction hypothesis, we use this to prove that $A_{j+1} \preceq u_{j+1} I$. Indeed, \cref{eq:uA-ub} implies that,
\begin{equation}
(u_j I - A_j) - (u_j I - A_{j+1}) = A_{j+1} - A_j \preceq \gamma (u_j I - A_j)
\end{equation}
Therefore,
\begin{equation}
    (1-\gamma) (u_j I - A_j ) \preceq u_j I - A_{j+1} \preceq u_{j+1} I - A_{j+1}
\end{equation}
And using the induction hypothesis that $u_j I - A_j \succeq 0$ completes the proof that $A_{j+1} \preceq u_{j+1} I$. On the other hand, to prove that $A_{j+1} \succeq l_{j+1} I$, summing \cref{eq:lA-lb} over all $x$ and noting that $\sum_x U(x) U(x)^T = I$,
\begin{equation} \label{eq:1234}
    \frac{1}{\Phi_j^{\text{Id}}} I \preceq A_j - l_j I
\end{equation}
Finally, observe that,
\begin{align}
    A_{j+1} - l_{j+1} I &= (A_j - l_j I) + \left( \frac{\gamma}{\Phi_j^{\text{Id}}} \frac{1}{p_j} U(x_j) U(x_j)^T - \frac{\gamma}{1+2\gamma} \frac{1}{\Phi_j^{\text{Id}}} I \right) \\
    &\succeq (A_j - l_j I) - \frac{\gamma}{1+2\gamma} \frac{1}{\Phi_j^{\text{Id}}} I \\
    &\overset{(i)}{\succeq} \frac{1}{\Phi_j^{\text{Id}}} I - \frac{\gamma}{1+2\gamma} \frac{1}{\Phi_j^{\text{Id}}} I \\
    &\succeq 0
\end{align}
where $(i)$ uses \cref{eq:1234}. 
\end{proof}

\end{document}